%% file: main.tex
\pdfoutput=1 

\documentclass[a4paper,11pt,reqno]{amsart}
\usepackage{lmodern}
\usepackage{hyperref}
\usepackage{times}
\usepackage{todonotes}
\usepackage{xcolor}
\usepackage{amsthm}                % better theorem environments
\usepackage{graphics}
\usepackage{graphicx}
\usepackage{latexsym}
\usepackage{tikz}
\usepackage{mathdots}

\usepackage{hyperref}
\usepackage{comment}

\usepackage[lite]{amsrefs}

\newtheorem{theorem}{Theorem}[section]

\newtheorem{proposition}[theorem]{Proposition}

\theoremstyle{definition}
\newtheorem{definition}[theorem]{Definition}

\theoremstyle{remark}
\newtheorem{remark}[theorem]{Remark}
\numberwithin{equation}{section}

\usepackage{wrapfig,lipsum,booktabs}
\usepackage{amssymb}
\usepackage{amsthm}    % better theorem environments
\usepackage{amsmath}
\usepackage{amsfonts}
\usepackage{comment}
\usepackage{enumerate}
\usepackage{graphicx}
\usepackage{wrapfig}

\usepackage{amssymb}

\usepackage{comment}
\usepackage{enumerate}

\usepackage{float}

\usepackage{graphics}
\usepackage{graphicx}
\usepackage{hyperref}

\usepackage{latexsym}
\usepackage{mathdots}
\usepackage{pict2e}
\usepackage[algoruled,linesnumbered,vlined,noline,noend]{algorithm2e}
\usepackage{MnSymbol} 

\begin{document}

\title{Parallel Mapper}

\author{Mustafa Hajij}
\address{KLA Corporation, USA}
\email{mustafahajij@gmail.com}

\author{Basem  Assiri} 
\address{Department of Computer Science, 
Jazan University, Jazan City, Saudi Arabia} 
\email{babumussmar@jazanu.edu.sa}

\author{Paul Rosen}
\address{University of South Florida, Tampa, FL 33620, U.S.A.}
\email{prosen@usf.edu}

\maketitle

\begin{abstract} The construction of Mapper has emerged in the last decade as a powerful and effective topological data analysis tool that approximates and generalizes other topological summaries, such as the Reeb graph, the contour tree, split, and joint trees. In this paper we study the parallel analysis of the construction of Mapper. We give a provably correct parallel algorithm to execute Mapper on a multiple processors and discuss the performance results that compare our approach to a reference sequential Mapper implementation. We report the performance experiments that demonstrate the efficiency of our method. 
\end{abstract}

\input{sec-topology_of_data}

\input{sec-motivating_example}

\input{sec-nerve}

\input{sec-applications}
%\input{higher_dim.tex}
%\input{further_directions}
\input{sec-alg_des}

\input{complexity}

\input{conclusion}

\bibliography{refs_2}

\end{document}

%% file: sec-topology_of_data.tex
\section{Introduction and Motivation}
\label{sec-topology_of_data}

The topology of data is one of the fundamental originating principle in studying data. Consider the classical problem of fitting data set of point in $\mathbb{R}^n$ using linear regression. In linear regression one usually assumes that data is almost distributed near a hyperplane in $\mathbb{R}^n$. See Figure \ref{linear} (a). If the data does not meet this assumption then the model chosen to fit the data may not work very well.

\begin{figure}[H]
  \centering
   {\includegraphics[width=9cm,height=5cm,keepaspectratio]{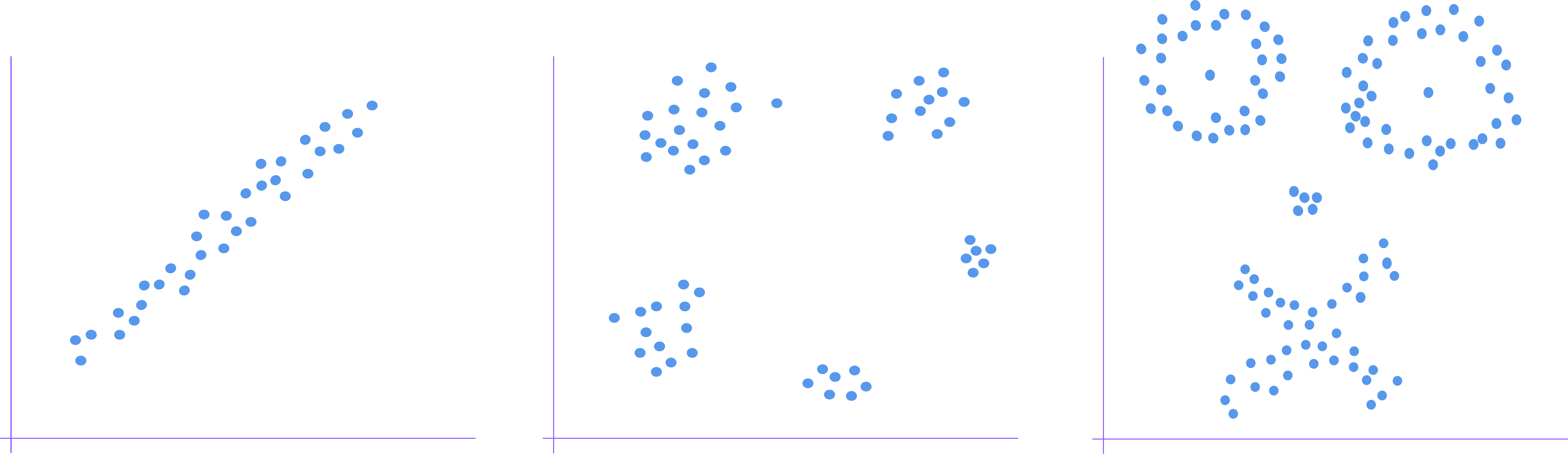}
          \put(-125,-10){(b)}
          \put(-35,-10){(c)}
           \put(-210,-10){(a)}
    \caption{(a) The linear shape of the data is a fundamental assumption underlying the linear regression method. (b) Clustering algorithms assume that the data is clustered in a certain way. (c) Data can come in many other forms and shapes.}
  \label{linear}}
\end{figure}

%\fix{this next paragraph reads awkwardly}
 On the other hand, a clustering algorithm normally makes the shape assumption that the data falls into clusters. See Figure \ref{linear} (b). Data can come in many other forms and shapes, see Figure  \ref{linear} (c). It is the \textit{shape of data} \cite{carlsson2009topology} that drives the meaning of these analytical methods and determines the successfulness of application of these methods on the data. 

%\fix{i fear that this introduction is not right given the audience. i would expect that they have at least a decent grasp on TDA. I would spend less time on this paragraph and more on why parallelizing this algorithm is an interesting/hard problem.}

Topology is the field in Mathematics that rigorously defines and studies the notion of shape. Over the past two decades, topology has found enormous applications in data analysis and the application of topological techniques to study data is now considered a vibrant area of research called as \textit{Topological Data Analysis} (TDA) \cite{carlsson2006algebraic,carlsson2009topology,carlsson2008local,carlsson2008persistent,carlsson2005persistence,carlsson2009theory,collins2004barcode}. Many popular tools have been invented in the last two decades to study the shape of data, most notably \textit{Persistent Homology}  \cite{edelsbrunner2000topological,robins1999towards} and the \textit{construction of Mapper} \cite{singh2007topological}. Persistent Homology has been successfully used to study a wide range of data problems including three-dimensional structure of the DNA \cite{emmett2016multiscale}, financial networks \cite{gidea2017topology}, material science \cite{hiraoka2016hierarchical} and many other applications \cite{otter2017roadmap}. The construction of  Mapper has emerged recently as a powerful and effective topological data analysis tool to solve a wide variety of problems \cite{lum2013extracting,nicolau2011topology,robles2017shape} and it has been studied from multiple points of view \cite{carriere2015structure,dey2017topological,munch2015convergence}. Mapper works as a tool of approximation of a topological space by mapping this space via a "lens", or a sometimes called a filter, to another domain. One uses properties of the lens and the codomain to then extract a topological approximation of the original space. We give the precious notion in Section \ref{sec:Preliminaries and Definitions}. Mapper generalizes other topological summaries such as the Reeb graph, the contour tree, split, and joint trees. Moreover, Mapper is the core software developed by Ayasdi, a data analytic company whose main interest is promoting the usage of methods inspired by topological constructions in data science applications.  

  As the demand of analyzing larger data sets grows, it is natural to consider parallelization of topological computations. While there are numerous parallel algorithms that tackle the less general topological constructions, such as Reeb graph and contour tree, we are not aware of similar attempts targeting the parallel computation of Mapper in the literature. Our work here is an attempt to fill in this gap.  

This article addresses the parallel analysis of the construction of Mapper. We give a provably correct algorithm to distribute Mapper on a set of processors and discuss the performance results that compare our approach to
a reference sequential implementation for the computation of Mapper. Finally, we report the performance analysis experiments that demonstrate the efficiency of our method.

%Our experimental analysis shows that our distributed Mapper algorithm provides about 4 time speedup over the sequential version of Mapper. 

%\noindent In summary, the contributions of this work are:
%\begin{itemize}
%\item A provably correct algorithm to distribute Mapper on a set of processors;
%\item An experimental analysis showing that our distributed Mapper algorithm provides about 4 time speedup over the sequential version of Mapper.
%\end{itemize}

%The standard algorithm for computing persistent homology assumes the input

%Approximation of topological spaces can be usually done in practice via simplicial complexes.  usually via  represented via simplicial complexes

%The purpose of this paper is to show how to split the computation of mapper into multiple processing units.    

% Some of the problems that we encounter when trying to analyze data using the classical methods is that we usually have no prior knowledge about the shape of the data at hand.
 
\section{Prior Work}

While there are numerous algorithms to compute topological constructions sequentially, the literature of parallel computing in topology is relatively young.  One notable exception is  parallelizing Morse-Smale complex computations \cite{gyulassy2012parallel,shivashankar2012parallel}. Parallelization of merge trees is studied in \cite{gueunet2017task,morozov2013distributed,pascucci2004parallel,rosen2017hybrid}. Other parallel algorithms in topology include multicore homology computation \cite{lewis2014multicore} spectral sequence parallelization \cite{lipsky2011spectral}, distributed contour tree \cite{morozov2012distributed}. There are several other attempts to speed up the serial computation of topological constructions including an optimized Mapper sequential algorithm for large data \cite{snavsel2017geometrical}, a memory efficient method to compute persistent cohomology \cite{boissonnat2015compressed},  efficient data structure for simplicial complexes \cite{bauer2017phat}, optimized computation of persistent homology \cite{chen2011persistent} and Morse-Smale complexes \cite{gunther2012efficient}.

%This paper is organized as follows. In section \ref{sec:Preliminaries and Definitions} we illustrate the concept of Mapper and give the definitions needed for the rest of the paper. In section \ref{Distributed Mapper Algorithm}, we give the topological construction behind the Distributed Mapper algorithm. In section \ref{main_algorithm}, we give the main algorithm and discuss its correctness. Finally, in section \ref{performance}, we discuss the performance of our algorithm. 
%Finally, in the appendix section \ref{appendix} we give the proof of the correctness of the main algorithm.

%% file: sec-motivating_example.tex
%\section{Preliminaries and Motivation}
\section{Preliminaries and Definitions}
\label{sec:Preliminaries and Definitions}
We start this section by recall basic notions from topology. For more details the reader is referred to standard texts in topology. See for instance \cite{munkres1984elements}. All topological spaces we consider in this paper will be compact unless otherwise specified. An \textit{open cover} of a topological space is a collection of open sets $\mathcal{U}=\{A_{\alpha} \}_{\alpha \in \mathcal{I} }$ such that $\cup_{\alpha \in \mathcal{I}} A_{\alpha}=X$. All covers in this article will consist of a finite number of sets unless otherwise specified. Given a topological space $X$ with a cover $\mathcal {U}$, one may approximate this space via an abstract simplicial complex construction called the \textit{nerve} of the cover $\mathcal{U}$. The nerve of a cover is a simplicial complex whose vertices are represented by the open sets the cover. Each non-empty intersection between two sets in the cover defines an edge in the nerve and each non-empty intersection between multiple sets defines higher order simplicies. See Figure \ref{nerve1212} for an illustrative example. Under mild conditions the nerve of a cover can be considered as an approximation of the underlying topological space. This is usually called the \textit{Nerve Theorem} \cite{ghrist2008barcodes}. The Nerve Theorem plays an essential role in TDA: it gives a mathematically justified approximation of the topological space, being thought as the data under study, via simplicial complexes which are suitable for data structures and algorithms. In \cite{singh2007topological} Singh et al proposed using a continuous map $f:X\longrightarrow Z$ to construct a nerve of the space $X$. Instead of covering $X$ directly, Singh et al suggested covering the codomain $Z$ and then use the map $f$ to pull back this cover to $X$. This perspective has multiple useful points of view. On one hand, choosing different maps  on $X$ can be used to capture different aspects of the space $X$. In this sense the function $f$ is thought of as a "lens" or a "filter" in which we view the space $X$. On the other hand, fixing the map $f$ and choosing  different covers for the codomain $Z$ can be used to obtain multi-level resolution of the Mapper structure. This has been recently studied in details in \cite{dey2016multiscale,dey2017topological} and utilized to obtain a notion of persistence-based signature based on the definition of Mapper.

The Mapper construction is related to Reeb graphs. To illustrate relationship, we give the following definition.

\begin{figure}[h]
  \centering
   {\includegraphics[width=12cm,height=6cm,keepaspectratio]{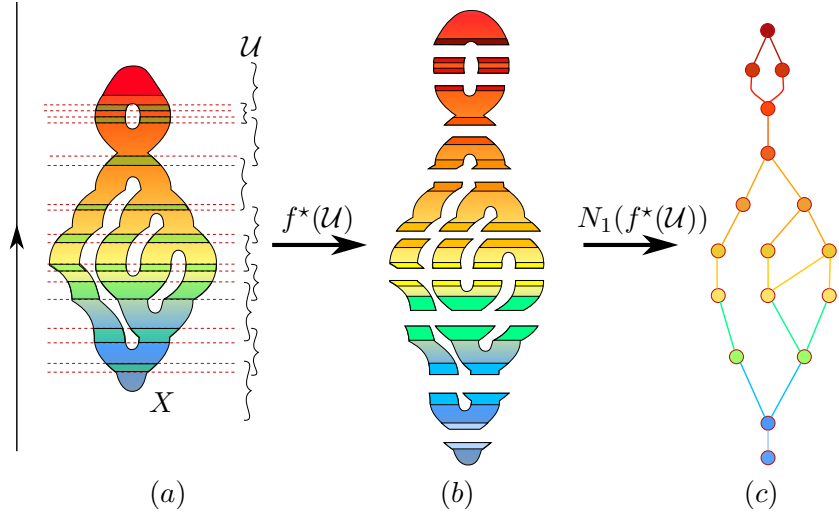}
      \put(-98,90){$N_1(f^{\star}(\mathcal{U}))$}
    \put(-210,90){$f^{\star}(\mathcal{U})$}
     \put(-225,155){$\mathcal{U}$}
     \put(-260,20){$X$}
     \put(-260,-14){$(a)$}
     \put(-150,-14){$(b)$}
     \put(-35,-14){$(c)$}
  }
  \caption{(a) Given a scalar function $f:X \longrightarrow [a,b]$ and an open cover $\mathcal{U}$ for $[a,b]$ we obtain an open cover $f^{\star}(\mathcal{U})$ for the space $X$ by considering the inverse images of the elements of $\mathcal{U}$ under $f$. (b) The connected-components of the inverse images are identified as well as the intersection between these sets. (c) Mapper is defined as a graph whose vertices represent the connected component and whose edge represent the intersection between these components.}
  \label{nerve_1}
\end{figure}

\begin{definition}
\label{nerve}
Let $X$ be a topological space and let $\mathcal{U}$ be an open cover for $X$. The $1$-\textbf{nerve} $N_1(\mathcal{U})$ of $\mathcal{U}$ is a graph whose nodes are represented by the elements of $\mathcal{U}$ and whose edges are the pairs ${A,B}$ of $\mathcal{U}$ such that $A\cap B \neq \emptyset$.
\end{definition}
\begin{figure}[h]
  \centering
   {\includegraphics[width=9cm,height=5cm,keepaspectratio]{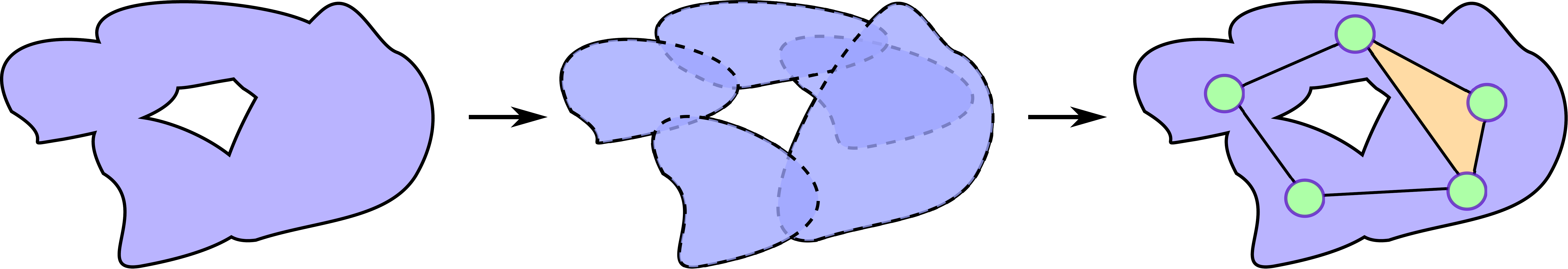}
   \put(-22,-4){$N(\mathcal{U})$}
    \put(-100,-4){$\mathcal{U}$}
     \put(-190,-4){$X$}
  }
  \caption{ Each open set defines a vertex in the nerve simplicial complex. Each intersection between two sets define an edge and intersection between multiple sets define higher order simplicies.}
%   \caption{Nerve of a topological space. Each open set defines a vertex in the nerve simplicial complex. Each intersection between two sets define an edge and intersection between multiple sets define higher order simplicies.}
  
  \label{nerve1212}
\end{figure}
A scalar function $f$ on $X$ and a cover for the codomain $[a,b]$ of $f$ give rise to a natural cover of $X$ in the following way. Start by defining an open cover for the interval $[a,b]$ and take the inverse image of each open set to obtain an open cover for $X$. This is illustrated in Figure \ref{nerve_1}~(a). In other words if $ \mathcal{U}=\{(a_1,b_1),...,(a_n,b_n)\}$ is a finite collection of open sets that covers the interval $[a,b]$ then $f^{\star}(\mathcal{U}):=\{f^{-1}((a_1,b_1)),...,f^{-1}((a_n,b_n))\}$ is an open cover for the space $X$.
\begin{comment}
\begin{figure}[h]
  \centering
   {\includegraphics[width=9cm,height=5cm,keepaspectratio]{mapper_1}
   \caption{An open cover $\mathcal{U}$ for the interval $[a,b]$ induces an open cover for the space $X$ by pulling the elements of $\mathcal{U}$ back via $f$. The green-shaded region on the surface represent the regions which the open sets of the cover overlap. }
  \label{open_cover}
  }
\end{figure}
\end{comment}
The open cover $f^{\star}(\mathcal{U})$ can now be used to obtain the $1$-nerve graph $N_1(f^{\star}(\mathcal{U}))$. With an appropriate choice of the cover $\mathcal{U}$, the graph $N_1(f^{\star}(\mathcal{U}))$ is a version of the Reeb graph $R(X,f)$ \cite{carriere2015structure,munch2015convergence}.  This is illustrated in Figure \ref{nerve_1}.

%
\begin{comment}
\begin{figure}[h]
  \centering
   {\includegraphics[width=9cm,height=5cm,keepaspectratio]{image_tda_all_geom_diff_colors1}
  }
  \caption{For the same function $f:X\longrightarrow [a,b]$, different choices of the cover $\mathcal{U}$ for $[a,b]$ induce different graphs. In other words, the $1$-nerve graph depends on the choice of the cover of $[a,b]$.  }
  \label{choices}
\end{figure}
\end{comment}

%% file: sec-nerve.tex
Observe that the different covers for $[a,b]$ give various ``resolution'' of the graph $N_1(f^{\star}(\mathcal{U}))$. The idea of mapper presented in Definition \ref{nerve} can be generalized to encompass a larger set of problems. One can replace the interval $[a,b]$ in Definition \ref{nerve} by any parametization domain $Z$ to obtain more sophisticated insights on the data $X$. This requires introducing the definition of a nerve of a cover of a topological space.
\begin{definition}
Let $X$ be a topological space and let $\mathcal{U}$ be a finite cover for $X$. The nerve of $\mathcal{U}$ is the abstract simplicial complex $N(\mathcal{U})$ whose vertices are the elements of  $\mathcal{U}$ and whose simplicies are the finite subcollections ${A_1,....,A_k}$ of  $\mathcal{U}$ such that : $A_1 \cap ... \cap A_k \neq\emptyset$.

\end{definition}
In this paper we will deal with nerves of multiple topological spaces simultaneously. For this reason we will sometimes refer to the nerve of a cover $\mathcal{U}$ of a space $X$ by $N(X,\mathcal{U})$. Figure \ref{nerve1212} shows an illustrative example of nerve on a topological space $X$. We will denote the vertex in $N(\mathcal{U})$ that corresponds to an open set $A$ in $\mathcal{U}$ by $v_{A}$.

\begin{comment}
Let $X$ be a topological space and let $\mathcal{A}$ and $\mathcal{B}$ be open covers for $X$. The cover $B$ is a \textit{refinement} $A$ if for each element of $B$ of $\mathcal{B}$ there is at least one element $A$ of $\mathcal{A}$ such that $B \subseteq  A$. If $\mathcal{A}$ and $\mathcal{B}$ are refinements of $X$ then one can show that there exists a simplicial map
\begin{equation*}
g:N(\mathcal{B})\longrightarrow N(\mathcal{A})
\end{equation*}
\end{comment}

%add this later with examples on the refinement

% If $\mathcal{B}$ is a refinement a cover $\mathcal{A}$, there is a embedding of the graph $N_1(X,\mathcal{A})$ inside the graph $N_1(X,\mathcal{B})$. That is there is one-to-one function $\phi$ that maps between the vertices sets $N_1(X,\mathcal{A})$ and $N_1(X,\mathcal{B})$ together with an assignment that assigns to every edge $e=(u,v)$ in $N_1(X,\mathcal{A})$ a path in $N_1(X,\mathcal{B})$ between $\phi(u)$ and $\phi(v)$.
Let $f:X\longrightarrow
Z$ be a continuous map between two topological spaces $X$ and $Z$. Let $\mathcal{U}$ be a  finite cover of $Z$. The cover that consists of $f^{-1}(U)$ for all open sets $U \in \mathcal{U}$ will be called the \textit{pullback} of $\mathcal{U}$ under $f$ and will be denoted by $f^{*}(\mathcal{U})$. A continuous map $f:X\longrightarrow
Z$ is said to be \textit{well-behaved} if the inverse image of any path-connected set $U$ in $Z$, consists of finitely many path-connected sets in $X$ \cite{dey2016multiscale}. All maps in this paper will be assumed to be well-behaved. %Furthermore, if $U$ is a path connected open set in $Z$ then $f^{-1}(U)$ might be decomposed into finite collection of open sets in many ways. To insure uniqueness, we will always consider the maximal path connected open sets that form $f^{-1}(U)$.
\begin{definition}
Let $f:X\longrightarrow Z$ be a continuous map between two topological space $X$ and $Z$. Let $\mathcal{U}$ be a finite cover for $Z$. The Mapper of $f$ and $\mathcal{U}$, denoted by $M(f,\mathcal{U})$, is the nerve $N(f^{*}{\mathcal{U}})$.
\end{definition}
%\begin{remark}
%Given a cover $\mathcal{U}$ in the codomain $Z$ in the previous cover, there might exist many ways to construct $f^{*}(\mathcal{U})$. This is because $f^{-1}(U)$ for an open set $U$ might be decmposied into a finite union of open sets in many way, For this reason if $U$ is a path connected open set in $Z$ then $f^{-1}(U)$ might be decomposed into finite collection of open sets in many ways. To insure uniqueness, we will always consider the maximal path connected open sets that form $f^{-1}(U)$ 
%\end{remark}

%For each set $X_i:=f^{-1}(a_i,b_i)$ its clusters $X_{ij} \subset X_i$ are found. Each cluster is considered as a vertex in the Mapper graph. Moreover we insert an edge between two nodes $X_{ij}$ and $X_{kl}$ whenever $X_{ij} \cap X_{ij} \neq \emptyset$. 

%The true power of this construction lies in its generality and flexibility. On one hand different choices of $f$ can be considered as different \textit{lenses} to study and analyze the space $X$. The possible lenses $f$ on $X$ are enormous \cite{singh2007topological}. On the other hand, for a fixed function $f:X\longrightarrow Z$ a hierarchy of cover refinements of $Z$ give rise to multi-resolutions mapper construction on $X$ \cite{dey2016multiscale}. This construction can be used to obtain a persistent Mapper-based signature of the space $X$.   
\subsection{Some Graph Theory Notions}
Our construction requires a few definitions from graph theory. We include these notions here for completeness. See \cite{beineke2004topics} for a more thorough treatment.
\begin{definition}
Let $G=(V,E)$ be a graph. Let $\sim$ be an equivalence relation defined on the node set $V$. The quotient graph of $G$ with respect to the equivalence relation is a graph $G/\sim$ whose node set is the quotient set $V/\sim$ and whose edge set is $\{ ([u],[v])| (u,v)\in E \}$. 
\end{definition}
\begin{wrapfigure}[6]{r}{0.49\textwidth}
	\centering
    \vspace{-10pt}
	\includegraphics[width=6cm,height=3cm,keepaspectratio]{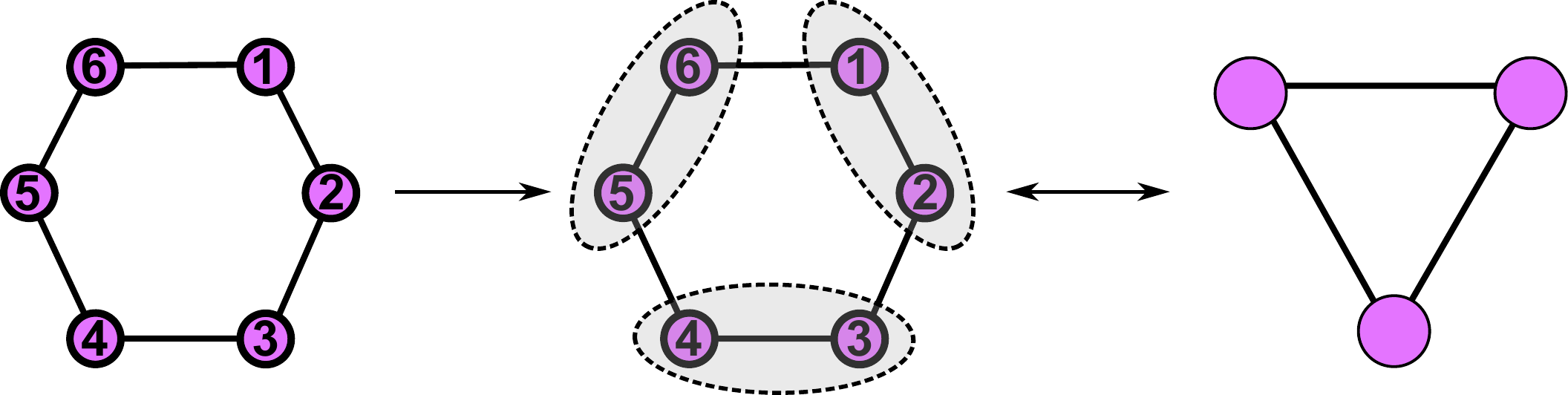}
	\caption{An example of a quotient graph.}
	\label{quotient graph}
\end{wrapfigure} 

For example consider the cyclic graph $C_6$ with $V=\{1,2,3,4,5,6\}$ and edges $(1,2)$,(2,3),...,$(6,1)$. Define the partition $\sim$ on $V$ by $p_1=\{1,2\}$, $p_2=\{3,4\}$ and $p_3=\{5,6\}$. The quotient graph induced by $\sim$ is the cyclic graph $C_3$. See Figure \ref{quotient graph}.

%\begin{figure}[H]
%  \centering
%   {\includegraphics[scale=0.2]{q_graphs}
%    \caption{An example of a quotient graph}
%  \label{quotient graph}}
%\end{figure}

We will also need the definition of disjoint union of two graphs. We will denote to the disjoint union of two sets $A$ and $B$ by $ A \sqcup B$.
\begin{definition}
Let $G_1=(V_1,E_1)$ and $G_2=(V_2,E_2)$ be two graphs. The disjoint union of $G_1$ and $G_2$ is the graph $G_1 \sqcup G_2$ defined by $(V_1 \sqcup V_2,E_1 \sqcup E_2)$.
\end{definition}
%Note that this is disjoint union operation on graphs is different the graph union operation.
%The main algorithm in this paper will utilize these two definitions. Namely, each processor $P_i$ will be assigned to compute a graph $G_i$. The disjoint union $\sqcup_i G_i$ will be quotient out by certain nodes to produce the final graph. 

%% file: sec-applications.tex
\section{Parallel Computing of Mapper}
\label{Parallel Mapper Algorithm}

%In this section we present a method to distribute the Mapper construction on two processing units. We present the idea using topological terms and then present the stasticial version in later section. 

%We start this section by discussing how the disjoint union of two nerves of subspaces of a space $X$ can be used to construct the nerve of the space $X$. More precisely, 

%For a space X with a covering $\mathcal{U}$, a vertex $v_A$ in $N_1(X,\mathcal{U})$ is induced by a open set $A$ in $\mathcal{U}$. Similarly, an edge $(v_A,v_B)$ is induced by an intersection $A\cap B$ between two open sets $A,B$ in $\mathcal{U}$. 
The idea of parallelizing the computation of Mapper lies in decomposing the space of interest into multiple smaller subspaces. The subspaces will be chosen to overlap on a smaller portion to insure a meaningful merging for the individual pieces. A cover of each space is then chosen. Each subspace along with its cover is then processed independently by a processing unit. The final stage consists of gathering the individual pieces and merging them together to produce the final correct Mapper construction on the entire space.

%\begin{wrapfigure}[10]{!h}{0.225\textwidth}
%	\centering
%    \vspace{-4pt}
%	{\includegraphics[width=0.975\linewidth]{distribuation}
%       \put(9,120){Unit $1$}
%   \put(9,44){Unit $2$}
%	\caption{Mapper on 2 Units}}
%	\label{dis}
%\end{wrapfigure} 
Let $f:X\longrightarrow [a,b]$ be a continuous function. The construction of parallel Mapper on two units goes as follows:

\begin{enumerate}

\item Choose an open cover for the interval $[a,b]$ that consists of exactly two sub-intervals $A_1$ and $A_2$ such that  $A:=A_1 \cap A_2 \neq \emptyset $.  See Figure \ref{dis_alg} (a).

\item

Choose open covers $\mathcal{U}_1$ and $\mathcal{U}_2$ for $A_1$ and $A_2$ respectively that satisfy the following conditions. First we want the intersection of the two coverings $\mathcal{U}_1$ and $ \mathcal{U}_1$ to have only the set $A$. Furthermore we do not want the covers $\mathcal{U}_1$ and $\mathcal{U}_2$ to overlap in anyway on any open set other than $A$.  %See Remark \ref{remark} below.

\item We compute the Mapper construction on the covers $f^{*}(\mathcal{U}_i)$ for $i=1,2$. We obtain two graphs $G_1$ and $G_2$. See Figure \ref{dis_alg} (b).
\item We merge the graphs $G_1$, $G_2$ as follows. By the construction of $A$, $\mathcal{U}_1$ and $\mathcal{U}_2$, the set $A$ exists in both covers $\mathcal{U}_i$, i=1,2. Let $C_1 ,..., C_n $ be the path-connected components of  $f^{-1}(A)$. Since $A$ appears in both of the covers then every connected component $C_i$ in $f^{-1}(A)$ occurs in both graphs $G_1$ and $G_2$. In other words, the nodes $v_1,...,v_n$ that correspond to the components $C_1 ,..., C_n $ occur in both $G_1$ and $G_2$  where each vertex $v_i$ corresponds to the set $C_i$. The merge of the graph is done by considering the disjoint union $G_1 \sqcup  G_2$ and then take the quotient of this graph by identifying the duplicate nodes $v_1,...,v_k$ presenting in both $G_1$ and $G_2$.  See Figure \ref{dis_alg} (c).
\end{enumerate}
The steps of the previous algorithm are summarized in Figure \ref{dis_alg}.
\begin{figure}[h]
  \centering
   {\includegraphics[width=9cm,height=5cm,keepaspectratio]{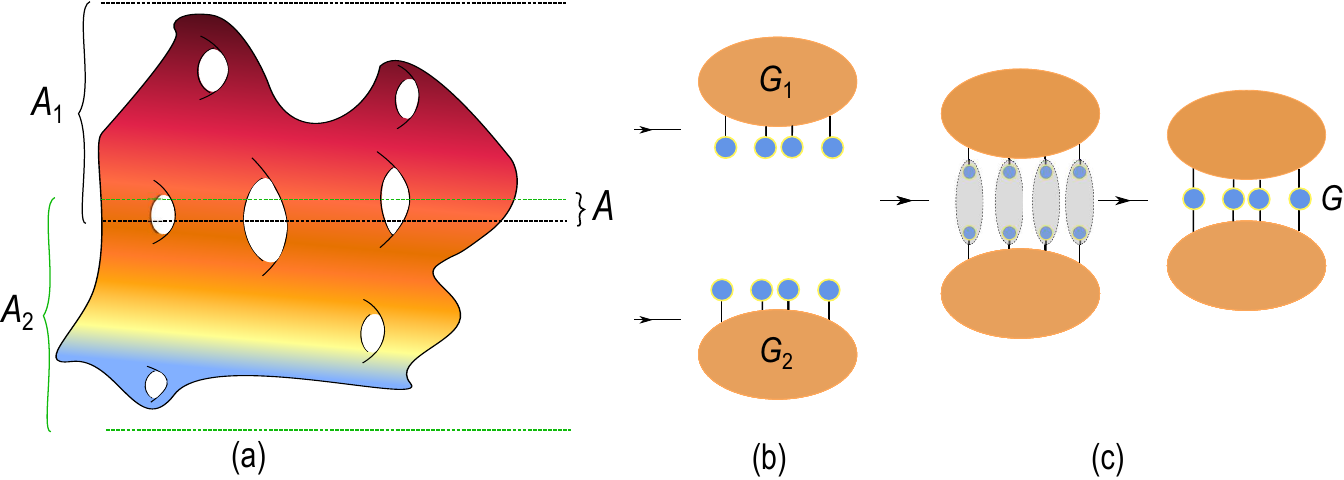}
   %\put(-390,99){$A_1$}
   %\put(-402,25){$A_2$}
   %\put(-225,60){$A$}
    \caption{The steps of the parallel Mapper on two units. (a) The space $X$ is decomposition based on a decomposition of the codomain (b) Each part is sent to a processing unit and the Mapper graphs are computed on the subspaces (c) The graphs are merged by identifying the corresponding the nodes.}
  \label{dis_alg}}
\end{figure}

%\begin{figure}[H]
%  \centering
%   {\includegraphics[scale=0.3]{distribuation_1}
%    \caption{Joining the two graphs.}
%  \label{graphf}}
%\end{figure}

\begin{remark}
Note that the interval $[a,b]$ in the construction above can be replaced by any domain $Y$ and the construction above remains valid. However for the purpose of this paper we restrict ourselves to the simplest case when $Y=[a,b]$. 
\end{remark}
 
Now define an $N$-chain cover of  $[a,b]$ to be a cover $\mathcal{U}$ of $[a,b]$ that consists of $N$ open intervals $A_1,...,A_N$ such that $A_{i, j}:=A_i \cap A_{j} \neq  \emptyset$ when $|i-j|=1$ and empty otherwise. By convention, a 1-chain cover for an interval $[a,b]$ is any open interval that contains $[a,b]$.

%The previous constructed can be generalized to decompose the codomain $Y$ into $N$  

%% file: sec-alg_des.tex
\section{The Design of the Algorithm}
\label{main_algorithm}
In this sectionو we discuss the computational details of the parallel Mapper algorithm that we already explained in the previous section from the topological perspective. Before we give our algorithm we recall quickly the reference sequential version.

\subsection{The Sequential Mapper Algorithm}
The serial Mapper algorithm can be obtained by a straightforward change of terminology of the topological mapper introduced in Section \ref{sec:Preliminaries and Definitions}. To this end, the topological space $X$ is replaced by the data under investigation. The lens, or the filter, $f$ is chosen to reflect a certain property of the data. Finally, the notion of path-connectedness is replaced by an appropriate notion of clustering. This is summarized in the Algorithm \ref{seq}. Note that we will refer the mapper graph obtained using Algorithm \ref{seq} by the \textit{sequential Mapper}.  

%See \cite{singh2007topological} for a detailed discussion on the appropriate conditions required in the clustering algorithm to correctly compute the Mapper graph.
\begin{algorithm}[h]
		%\dontprintsemicolon
%		\vspace{1mm}		
		\KwIn{A dataset $X$ with a notion of metric between the data points\; a scalar function $f :X\longrightarrow \mathbb{R}^n$\; a finite cover $\mathcal{U}=\{U_1,...,U_k\} $ of $f(X)$;}
		\KwOut{ A graph that represents $N_1(f^{\star}(\mathcal{U}))$.} 
		\vspace{1mm}
For each set $X_i:=f^{-1}(U_i)$, its clusters $X_{ij} \subset X_i$ are computed using the chosen clustering algorithm.\;
Each cluster is considered as a vertex in the Mapper graph. Moreover we insert an edge between two nodes $X_{ij}$ and $X_{kl}$ whenever $X_{ij} \cap X_{kl} \neq \emptyset$;
		\caption{Sequential Mapper \cite{singh2007topological}}
        \label{seq}
\end{algorithm}

\subsection{The Main Algorithm}
We now give the details of the parallel Mapper algorithm.
To guarantee that the output of the parallel Mapper is identical to that of the sequential Mapper we need to perform some processing on the cover that induces the final parallel Mapper output. In parallel Mapper, we consider an $N$-chain cover of open intervals $A_1,\cdots, A_N$ of the interval $[a,b]$ along with the their covers $\mathcal{U}_1$,...,$\mathcal{U}_N$. The details of the cover preprocessing are described in Algorithm \ref{preprocessing}.

%The justification of the previous choice of covers will be given when we discuss the correctness of Algorithm \ref{mapper on N}.

\begin{algorithm}[h]
		%\dontprintsemicolon
%		\vspace{1mm}		
		\KwIn{A point cloud $X$\;
        a scalar function $f :X\longrightarrow [a,b]$\;
        a set of $N$ processors ($\mathcal{P}$)\;}
		\KwOut{ A collection of pairs $\{(A_i,\mathcal{U}_i)\}_{i=1}^N$ where $\{A_i\}_{i=1}^N$ is an $N$-chain cover of $[a,b]$ and $\mathcal{U}_i$ is a cover of $A_i$.} 
		\vspace{1mm}
Construct an $N$-chain cover of  $[a,b]$. That is, cover $[a,b]$ by $N$ open intervals $A_1,\cdots, A_N$ such that $A_{i, j}:=A_i \cap A_{j} \neq  \emptyset$ when $|i-j|=1$ and empty otherwise\;
For each open set $A_i$ construct an open cover $\mathcal{U}_i$. The covers $\{\mathcal{U}_i \}_{i=1}^N$ satisfy the following conditions: (1) $A_{i,i+1}$ is an open set in both coverings $\mathcal{U}_i$ and $\mathcal{U}_{i+1}$. In other words $\mathcal{U}_i \cap \mathcal{U}_{i+1}=\{A_{i, i+1}\}$ and (2) if  $U_i\in \mathcal{U}_i$ and $U_{i+1}\in \mathcal{U}_{i+1}$ such that $U_i\cap U_{i+1} \neq \emptyset$ then $U_i\cap U_{i+1}=A_{i,i+1}$ for each $i=1,...,N-1$\;
		\caption{Cover Preprocessing}
        \label{preprocessing}
\end{algorithm}

\begin{algorithm}[H]
		%\dontprintsemicolon
%		\vspace{1mm}		
		\KwIn{A point cloud $X$\;
        a scalar function $f :X\longrightarrow [a,b]$\;
        a set of $N$ processors ($\mathcal{P}$)\;
a collection of pairs $\{(A_i,\mathcal{U}_i)\}_{i=1}^N$ obtained from the cover preprocessing algorithm\;       
        }
		\KwOut{Parallel Mapper Graph.} 
        %where proper identification and gluing of perimeter edges yields $M$
		\vspace{1mm}
 %        Map each $A_i$ to a specific Processor $P$\;
            \For{( $i \gets 1$ to $i=N$)}
{$P_i \gets (A_i,\mathcal{U}_i)$; //Map each $A_i$, and its cover $\mathcal{U}_i$ to the processor $P_i.$}
Determine the set of point $X_i \subset X$ that maps to $A_i$ via $f$ and run the sequential Mapper construction concurrently on the covers $(f|_{X_i})^{*}(\mathcal{U}_i)$ for $i=1,..,N$. We obtain $N$ graphs $G_1,...G_N$. If $N=1$, return the graph $G_1$\;
Let $C^i_{j_1},...,C^i_{j_i} $ be the  clusters obtained from  $f^{-1}(A_{i, i+1})$. These clusters are represented by the vertices $v^i_{j_1},...,v^i_{j_i}$ in both $G_i$ and $G_{i+1}$ (each vertex $v^i_k$ corresponds to the cluster $C^i_k$) by the choice of the coverings $\mathcal{U}_i$ and $\mathcal{U}_{i+1}$\;
Merge the graphs $G_1,..., G_N$ as follows. By the construction of $A_{i, i+1}$, $\mathcal{U}_i$ and $\mathcal{U}_{i+1}$, each one of the sets $f^*(\mathcal{U}_i)$ and $f^*(\mathcal{U}_{i+1})$ share the clusters $C^i_{j_k}$ in $f^*(A_{i,i+1})$ . Hence $C^i_{j_k}$ is represented by a vertex in both graphs $G_i$ and $G_{i+1}$. The merging is done by considering the disjoint union graph $G_1 \sqcup ...\sqcup G_N$ and then take the quotient of this graph that identifies the corresponding vertices in $G_i$ and $G_{i+1}$ for $1\leq i \leq N-1$.	       
		\caption{Parallel Mapper}
        \label{mapper on N}
\end{algorithm}

After doing the preprocessing of the cover and obtaining the collection $\{(A_i,\mathcal{U}_i)\}_{i=1}^N$, every pair $(A_i,\mathcal{U}_i$) is mapped to a specific processor $P_i$ which performs some calculations to produce a subgraph $G_i$. At the end, we merge the subgraphs into one graph $G$. The details of the algorithm are presented in Algorithm \ref{mapper on N}.

\subsection{Correctness of the Algorithm}

%The correctness of the previous algorithm basically follows Corollary \ref{main corollary}. 
In here, we give a detailed proof of the correctness of parallel Mapper that discusses the steps of the algorithm.

\begin{proposition}
\label{correctness}
The parallel Mapper algorithm returns a graph identical to the sequential Mapper.
\end{proposition}
\begin{proof}

We will prove that the parallel Mapper performs the computations on $X$ and correctly produces a graph $G$ that is identical to the graph obtained by the sequential Mapper algorithm using induction.

Denote by $N$ to the number of units of initial partitions of interval $I$, which is the same number of processing units. If $N = 1$, then the parallel Mapper works exactly like the sequential Mapper. In this case $A_1=X$ and the single cover $\mathcal{U}_1$ for $X$ is used to produce the final graph which  Algorithm \ref{mapper on N} returns at step $(3)$.

%Suppose that $N=2$. In this case we have two processors $P_1$ and $P_2$. Step 2 in the algorithm constructs  a 2-chain cover $\{A_1,A_2\}$ on the codomain $I$. Each set $A_i$ is then covered with a cover $\mathcal{U}_i$.  

Now assume the hypothesis is true on $k$ unit, and then we show that it holds on $k+1$ units. In step $(1)$ and $(2)$ Algorithm \ref{mapper on N} constructs a $k+1$-chain cover for $[a,b]$ consisting of the open sets $A_1,...,A_k,A_{k+1}$. Denote by $\mathcal{U}_i$ to the cover of $A_i$ for $1\leq i\leq k+1 $. We can run Algorithm \ref{mapper on N} on the collection $\{(A_i,\mathcal{U}_i)\}_{i=1}^{k}$ and produce a sequential Mapper graphs $G_i$ $1\leq i \leq k $ in step (3). By the induction hypothesis, Algorithm \ref{mapper on N} produces correctly a graph $G^{\prime}$ obtained by merging the sequential Mapper graphs $G_1,...,G_k$. In other words the graph $G^{\prime}$ obtained from Algorithm \ref{mapper on N} is identical to the graph obtain by running the sequential Mapper construction on the cover $\cup_i^{k} \mathcal{U}i$.

Now we show that combining $G^{\prime}$ and $G_{k+1}$ using our algorithm produces a graph $G$ that is identical to running the sequential Mapper on the covering consists of $\cup_i^{k+1} \mathcal{U}i$. Let $\mathcal{U}^{\prime}$ be the union $\cup_i^{k} \mathcal{U}i$ and a denote by $A^{\prime}$ to the union $\cup_{i=1}^k A_i$. By the  construction of the covers $\{\mathcal{U}_i\}_{i=1}^{k+1}$ in step $(2)$, $\mathcal{U}^{\prime}$ covers $A^{\prime}$. Moreover, the covers $\mathcal{U}^{\prime}$ and $\mathcal{U}_{k+1}$ only share the open set $A^{\prime} \cap A_{k+1}$. This means there are no intersections between the open sets of the cover $\mathcal{U}^{\prime}$ and the open sets of the cover $\mathcal{U}_{k+1}$ except for $A^{\prime} \cap A_{k+1}$. Since there is no intersection between the open sets of $\mathcal{U}^{\prime}$ and $\mathcal{U}_{k+1}$ then there will be no creation of edges between the nodes induced from them and hence the computation of edges done on the first $k$ processors are independent from the computation of edges done on the $k+1$ possessor. Now we compare the node sets of the graphs $G^{\prime}$, $G_{k+1}$ and the graph $G$. Recall that each node in a sequential Mapper is obtained by a connected component of an inverse image of an open set in the cover that defines the Mapper construction. Since the covers $\mathcal{U}^{\prime}$ and $\mathcal{U}_{k+1}$ intersect at the open set $f^{-1}(A^{\prime}\cap A_{k+1})$ then each connected component of $f^{-1}(A^{\prime}\cap A_{k+1})$ corresponds to a node that exists in both graphs $G^{\prime}$ and $G_{k+1}$. %Moreover, there are o   is obtained by taking an open On the other hand, the cover consists of $\mathcal{U}^{\prime} \cup \mathcal{U}_1  $ induces on the sequential Mapper corresponding to each connected component of $f^{-1}(A^{\prime}\cap A_{k+1})$. 
This means that each connected component of $f^{-1}(A^{\prime}\cap A_{k+1})$ is processed twice : one time on the first $k$ processor and one time on the $k+1$ processors.For each such component corresponds to a node in both $G^{\prime}$ and $G_{k+1}$. In step $(5)$ the algorithm checks the graphs $G^{\prime}$ and $G_{k+1}$ for node duplication and merge them according to their correspondence to produce the graph $G$.
\end{proof}

%% file: complexity.tex
\section{Experimentations}

In this section, we present practical results obtained using a Python implementation. We ran our experimentation on a Dell OptiPlex 7010 machine with 4-core i7-3770 Intel CPU @ 3.40GHz and with a 24GiB System Memory. The parallel Mapper algorithm was tested on different models and compared their run-time with a
publicly available data available at \cite{sumner2004deformation}. The size of the point cloud data are shown in Table~\ref{wrap-tab:1}. The size of datasets given in Table \ref{wrap-tab:1} is the number of points in the point cloud data.

\begin{table}[h!]
\centering
\begin{tabular}{ |p{1.2cm}||p{1.2cm}| }
 \hline
 Data &  Size\\
 \hline
 camel &21887\\
 cat & 7277 \\
 elephant & 42321 \\
 horse & 8431 \\
  face & 29299 \\
  head & 15941 \\
 \hline
\end{tabular}
\caption{The number of points for each dataset used in our tests.}
\label{wrap-tab:1}
\end{table}
%camel=21887, cat=7277, elephant=42321, horse=8431, face=29299, head=15941.

The sequential Mapper algorithm relies on $3$ inputs: the data $X$, the scalar function $f:X \longrightarrow [a,b]$ and the choice of cover $\mathcal{U}$ of $[a,b]$. The existing publicly available Mapper implementations, see for instance \cite{mullner2013python}, do not satisfy the level of control that we require for the cover choice and so we relied on our own Mapper implementation. The clustering algorithm that we used to specify the Mapper nodes is a modified version of the DBSCAN \cite{calinski1974dendrite}. 
\begin{comment}
\begin{table}[16]{r}{3cm}
\begin{tabular}{cc}  
\vspace{0pt}
Data & Size \\\midrule
camel &21887 pt\\  \midrule
cat &7207 pt\\  \midrule
elephant &42321 pt \\  \midrule
horse &8431 pt\\  \bottomrule
face &29299 pt\\  \bottomrule
head &15941 pt\\  \bottomrule
\end{tabular}
\vspace{0pt}
\caption{The number of points for each dataset used in our tests.}\label{wrap-tab:1}
\end{table}
\end{comment}
Using parallel Mapper on the data given in Table \ref{wrap-tab:1}, we obtained a remarkable speed up that can be $4$ times faster, compared with the sequential Mapper. Figure \ref{speedup}, shows the speedup results of parallel Mapper that are obtained using our experiments. The $x$-axis represents the number of processes while the $y$-axis shows the speedup. It is clear from the figure that the curves are increasing in a monotonic fashion as we increase the number of processes. Indeed, at the beginning the speed up increases significantly as we increase the number of processes. However, at some point (when we use more than 10 processes), we increase the number of processes to 30 processes and the speedup does not show significant improvement. 

\begin{figure}[h]%[12]{r}{0.425\textwidth}
	\centering
    %\vspace{-10pt}
	%\includegraphics[width=0.8\linewidth]{speedup} 
    \includegraphics[width=9cm,height=5cm,keepaspectratio]{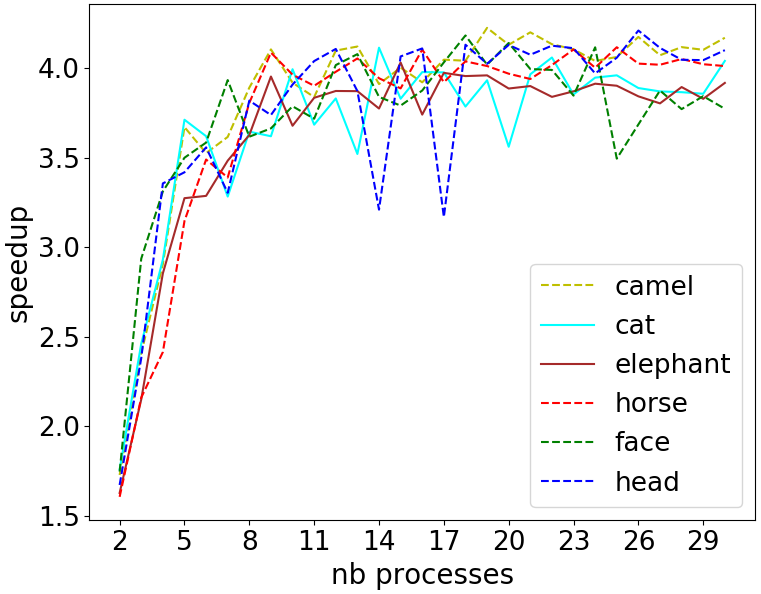}
	\caption{Speedups obtained by the parallel Mapper using number of processes that run concurrently.}
	\label{speedup}
\end{figure}

\begin{comment}
\begin{figure}[H]
  \centering
   {\includegraphics[scale=0.237]{allr}
   %\put(-390,99){$A_1$}
   %\put(-402,25){$A_2$}
   %\put(-225,60){$A$}
    \caption{horse,face,head,elephant,cat camel.}
  \label{results}}
\end{figure}
\end{comment}

%==============================================
\subsection{Performance of the Algorithm}
\label{performance}

To verify our experimental results in Figure \ref{speedup}, we use a well-known theoretical formula which is the \textit{Amedahl's law} to calculate the {\em speedup} ratio upper bound that comes from parallelism and the improvement percentage \cite{amdahl1967validity}. The Amedahl's law is formulated as follows: 

{
\centering{${ \displaystyle
S = {\frac {1}{(1-part)+ part/N}}}$},\\
}
where $S$ is the theoretical speedup ratio, $part$ is the proportion of system or program that can be made in parallel, $1-part$ is the proportion that remains sequential, and $N$ is the number of processes.

Generally, there are some systems and applications where parallelism cannot be applied on all data or processes. In this case, part of data can be processed in parallel, while the other should be sequential. This may happen because of the nature of data (e.g. dependencies), the natures of processes (e.g. heterogeneity) or some other factors.

In the parallel Mapper, there are two computational pieces which are the clustering piece and the cover construction/merging subgraphs piece. Our algorithm makes the clustering piece completely in parallel while the cover construction/merging subgraphs piece is processed sequentially. Now, we use Amedahl's law to calculate the theoretic speedup ratios to verify the experimental results. Indeed, considering the algorithm, the clustering piece is approximately 75\% of execution time, while the cover construction/merging subgraphs piece is about 25\%.

% our execution time experimental analysis showed that clustering piece is approximately 75\% of execution time, while the cover construction/merging subgraphs piece is about 25\% of the execution time.

In Table \ref{table:1}, we use Amedahl's law to calculate the theoretic speedup ratios using different numbers of processes. The table shows that the speedup increases as a response of the increase in the number of processes. Notice that at some points the performance almost stops improving even if we increase the number of processes. Table \ref{table:1} shows that the speedup of $part = .75$ (the parallel Mapper) achieves to $3.07$ when $N=10$ and it goes up to $3.99$ when $N=1000$. Therefore, the theoretical calculations clearly matches the experimental results that appears in Figure \ref{speedup}.

\begin{table}[H]
\centering
\begin{tabular}{ |p{1cm}||p{2.5cm}| }
 \hline
 \multicolumn{2}{|c|}{Speedup} \\
 \hline
 N &  \textsf{(Parallel Mapper)} part=0.75 \\
 \hline
 10 & 3.07 \\
 100 & 3.88 \\
 1000 & 3.99 \\
 10000 & 3.99 \\
 \hline
\end{tabular}
 \vspace{1pt}
\caption{Speedup calculations based on Amedahl's law, using different numbers of processes. It shows the speedup of the parallel Mapper with respect to the sequential Mapper}
\label{table:1}
\end{table}

%% file: conclusion.tex
\section{Conclusion and Future Work}
\label{future}
In this work we gave a provably correct algorithm to distribute Mapper on a set of processors and run them in parallel. Our algorithm relies on a divide an conquer strategy for the codomain cover which gets pulled back to the domain cover. This work has several potential directions of the work that we have not discussed here. For instance, the recursive nature of the main algorithm was implied throughout the paper but never discussed explicitly. On the other hand the algorithm can be utilized to obtain a multi-resolution Mapper construction. In other words, using this algorithm we have the ability to increase the resolution of Mapper for certain subsets of the data and decrease at others. This is potentially useful for interactive Mapper applications. 

%% file: main.bbl
\begin{bibdiv}
\begin{biblist}

\bib{amdahl1967validity}{inproceedings}{
      author={Amdahl, Gene~M},
       title={Validity of the single processor approach to achieving large
  scale computing capabilities},
organization={ACM},
        date={1967},
   booktitle={Proceedings of the april 18-20, 1967, spring joint computer
  conference},
       pages={483\ndash 485},
}

\bib{bauer2017phat}{article}{
      author={Bauer, Ulrich},
      author={Kerber, Michael},
      author={Reininghaus, Jan},
      author={Wagner, Hubert},
       title={Phat--persistent homology algorithms toolbox},
        date={2017},
     journal={Journal of Symbolic Computation},
      volume={78},
       pages={76\ndash 90},
}

\bib{beineke2004topics}{book}{
      author={Beineke, Lowell~W},
      author={Wilson, Robin~J},
       title={Topics in algebraic graph theory},
   publisher={Cambridge University Press},
        date={2004},
      volume={102},
}

\bib{boissonnat2015compressed}{article}{
      author={Boissonnat, Jean-Daniel},
      author={Dey, Tamal~K},
      author={Maria, Cl{\'e}ment},
       title={The compressed annotation matrix: An efficient data structure for
  computing persistent cohomology},
        date={2015},
     journal={Algorithmica},
      volume={73},
      number={3},
       pages={607\ndash 619},
}

\bib{calinski1974dendrite}{article}{
      author={Cali{\'n}ski, Tadeusz},
      author={Harabasz, Jerzy},
       title={A dendrite method for cluster analysis},
        date={1974},
     journal={Communications in Statistics-theory and Methods},
      volume={3},
      number={1},
       pages={1\ndash 27},
}

\bib{carlsson2006algebraic}{article}{
      author={Carlsson, Erik},
      author={Carlsson, Gunnar},
      author={De~Silva, Vin},
       title={An algebraic topological method for feature identification},
        date={2006},
     journal={International Journal of Computational Geometry \& Applications},
      volume={16},
      number={04},
       pages={291\ndash 314},
}

\bib{carlsson2009topology}{article}{
      author={Carlsson, Gunnar},
       title={Topology and data},
        date={2009},
     journal={Bulletin of the American Mathematical Society},
      volume={46},
      number={2},
       pages={255\ndash 308},
}

\bib{carlsson2008local}{article}{
      author={Carlsson, Gunnar},
      author={Ishkhanov, Tigran},
      author={De~Silva, Vin},
      author={Zomorodian, Afra},
       title={On the local behavior of spaces of natural images},
        date={2008},
     journal={International journal of computer vision},
      volume={76},
      number={1},
       pages={1\ndash 12},
}

\bib{carlsson2008persistent}{article}{
      author={Carlsson, Gunnar},
      author={M{\'e}moli, Facundo},
       title={Persistent clustering and a theorem of j. kleinberg},
        date={2008},
     journal={arXiv preprint arXiv:0808.2241},
}

\bib{carlsson2009theory}{article}{
      author={Carlsson, Gunnar},
      author={Zomorodian, Afra},
       title={The theory of multidimensional persistence},
        date={2009},
     journal={Discrete \& Computational Geometry},
      volume={42},
      number={1},
       pages={71\ndash 93},
}

\bib{carlsson2005persistence}{article}{
      author={Carlsson, Gunnar},
      author={Zomorodian, Afra},
      author={Collins, Anne},
      author={Guibas, Leonidas~J},
       title={Persistence barcodes for shapes},
        date={2005},
     journal={International Journal of Shape Modeling},
      volume={11},
      number={02},
       pages={149\ndash 187},
}

\bib{carriere2015structure}{article}{
      author={Carri{\`e}re, Mathieu},
      author={Oudot, Steve},
       title={Structure and stability of the 1-dimensional mapper},
        date={2015},
     journal={arXiv preprint arXiv:1511.05823},
}

\bib{chen2011persistent}{inproceedings}{
      author={Chen, Chao},
      author={Kerber, Michael},
       title={Persistent homology computation with a twist},
        date={2011},
   booktitle={Proceedings 27th european workshop on computational geometry},
      volume={11},
}

\bib{collins2004barcode}{article}{
      author={Collins, Anne},
      author={Zomorodian, Afra},
      author={Carlsson, Gunnar},
      author={Guibas, Leonidas~J},
       title={A barcode shape descriptor for curve point cloud data},
        date={2004},
     journal={Computers \& Graphics},
      volume={28},
      number={6},
       pages={881\ndash 894},
}

\bib{dey2016multiscale}{inproceedings}{
      author={Dey, Tamal~K},
      author={M{\'e}moli, Facundo},
      author={Wang, Yusu},
       title={Multiscale mapper: topological summarization via codomain
  covers},
organization={Society for Industrial and Applied Mathematics},
        date={2016},
   booktitle={Proceedings of the twenty-seventh annual acm-siam symposium on
  discrete algorithms},
       pages={997\ndash 1013},
}

\bib{dey2017topological}{article}{
      author={Dey, Tamal~K},
      author={Memoli, Facundo},
      author={Wang, Yusu},
       title={Topological analysis of nerves, reeb spaces, mappers, and
  multiscale mappers},
        date={2017},
     journal={arXiv preprint arXiv:1703.07387},
}

\bib{edelsbrunner2000topological}{inproceedings}{
      author={Edelsbrunner, Herbert},
      author={Letscher, David},
      author={Zomorodian, Afra},
       title={Topological persistence and simplification},
organization={IEEE},
        date={2000},
   booktitle={Foundations of computer science, 2000. proceedings. 41st annual
  symposium on},
       pages={454\ndash 463},
}

\bib{emmett2016multiscale}{inproceedings}{
      author={Emmett, Kevin},
      author={Schweinhart, Benjamin},
      author={Rabadan, Raul},
       title={Multiscale topology of chromatin folding},
organization={ICST (Institute for Computer Sciences, Social-Informatics and
  Telecommunications Engineering)},
        date={2016},
   booktitle={Proceedings of the 9th eai international conference on
  bio-inspired information and communications technologies (formerly
  bionetics)},
       pages={177\ndash 180},
}

\bib{ghrist2008barcodes}{article}{
      author={Ghrist, Robert},
       title={Barcodes: the persistent topology of data},
        date={2008},
     journal={Bulletin of the American Mathematical Society},
      volume={45},
      number={1},
       pages={61\ndash 75},
}

\bib{gidea2017topology}{article}{
      author={Gidea, Marian},
       title={Topology data analysis of critical transitions in financial
  networks},
        date={2017},
}

\bib{gueunet2017task}{inproceedings}{
      author={Gueunet, Charles},
      author={Fortin, Pierre},
      author={Jomier, Julien},
      author={Tierny, Julien},
       title={Task-based augmented merge trees with fibonacci heaps},
        date={2017},
   booktitle={Ieee symposium on large data analysis and visualization 2017},
}

\bib{gunther2012efficient}{article}{
      author={G{\"u}nther, David},
      author={Reininghaus, Jan},
      author={Wagner, Hubert},
      author={Hotz, Ingrid},
       title={Efficient computation of 3d morse--smale complexes and persistent
  homology using discrete morse theory},
        date={2012},
     journal={The Visual Computer},
      volume={28},
      number={10},
       pages={959\ndash 969},
}

\bib{gyulassy2012parallel}{inproceedings}{
      author={Gyulassy, Attila},
      author={Pascucci, Valerio},
      author={Peterka, Tom},
      author={Ross, Robert},
       title={The parallel computation of morse-smale complexes},
organization={IEEE},
        date={2012},
   booktitle={Parallel \& distributed processing symposium (ipdps), 2012 ieee
  26th international},
       pages={484\ndash 495},
}

\bib{hiraoka2016hierarchical}{article}{
      author={Hiraoka, Yasuaki},
      author={Nakamura, Takenobu},
      author={Hirata, Akihiko},
      author={Escolar, Emerson~G},
      author={Matsue, Kaname},
      author={Nishiura, Yasumasa},
       title={Hierarchical structures of amorphous solids characterized by
  persistent homology},
        date={2016},
     journal={Proceedings of the National Academy of Sciences},
      volume={113},
      number={26},
       pages={7035\ndash 7040},
}

\bib{lewis2014multicore}{article}{
      author={Lewis, Ryan~H},
      author={Zomorodian, Afra},
       title={Multicore homology via mayer vietoris},
        date={2014},
     journal={arXiv preprint arXiv:1407.2275},
}

\bib{lipsky2011spectral}{article}{
      author={Lipsky, David},
      author={Skraba, Primoz},
      author={Vejdemo-Johansson, Mikael},
       title={A spectral sequence for parallelized persistence},
        date={2011},
     journal={arXiv preprint arXiv:1112.1245},
}

\bib{lum2013extracting}{article}{
      author={Lum, PY},
      author={Singh, G},
      author={Lehman, A},
      author={Ishkanov, T},
      author={Vejdemo-Johansson, Mikael},
      author={Alagappan, M},
      author={Carlsson, J},
      author={Carlsson, G},
       title={Extracting insights from the shape of complex data using
  topology},
        date={2013},
     journal={Scientific reports},
      volume={3},
       pages={1236},
}

\bib{morozov2013distributed}{inproceedings}{
      author={Morozov, Dmitriy},
      author={Weber, Gunther},
       title={Distributed merge trees},
organization={ACM},
        date={2013},
   booktitle={Acm sigplan notices},
      volume={48},
       pages={93\ndash 102},
}

\bib{morozov2012distributed}{article}{
      author={Morozov, Dmitriy},
      author={Weber, Gunther~H},
       title={Distributed contour trees},
        date={2012},
}

\bib{mullner2013python}{article}{
      author={M{\"u}llner, Daniel},
      author={Babu, Aravindakshan},
       title={Python mapper: An open-source toolchain for data exploration,
  analysis, and visualization},
        date={2013},
     journal={URL http://math. stanford. edu/muellner/mapper},
}

\bib{munch2015convergence}{article}{
      author={Munch, Elizabeth},
      author={Wang, Bei},
       title={Convergence between categorical representations of reeb space and
  mapper},
        date={2015},
     journal={arXiv preprint arXiv:1512.04108},
}

\bib{munkres1984elements}{book}{
      author={Munkres, James~R},
       title={Elements of algebraic topology},
   publisher={Addison-Wesley Menlo Park},
        date={1984},
      volume={2},
}

\bib{nicolau2011topology}{article}{
      author={Nicolau, Monica},
      author={Levine, Arnold~J},
      author={Carlsson, Gunnar},
       title={Topology based data analysis identifies a subgroup of breast
  cancers with a unique mutational profile and excellent survival},
        date={2011},
     journal={Proceedings of the National Academy of Sciences},
      volume={108},
      number={17},
       pages={7265\ndash 7270},
}

\bib{otter2017roadmap}{article}{
      author={Otter, Nina},
      author={Porter, Mason~A},
      author={Tillmann, Ulrike},
      author={Grindrod, Peter},
      author={Harrington, Heather~A},
       title={A roadmap for the computation of persistent homology},
        date={2017},
     journal={EPJ Data Science},
      volume={6},
      number={1},
       pages={17},
}

\bib{pascucci2004parallel}{article}{
      author={Pascucci, Valerio},
      author={Cole-McLaughlin, Kree},
       title={Parallel computation of the topology of level sets},
        date={2004},
     journal={Algorithmica},
      volume={38},
      number={1},
       pages={249\ndash 268},
}

\bib{robins1999towards}{inproceedings}{
      author={Robins, Vanessa},
       title={Towards computing homology from finite approximations},
        date={1999},
   booktitle={Topology proceedings},
      volume={24},
       pages={503\ndash 532},
}

\bib{robles2017shape}{article}{
      author={Robles, Alejandro},
      author={Hajij, Mustafa},
      author={Rosen, Paul},
       title={The shape of an image: A study of mapper on images},
        date={2018},
     journal={to appear VISAPP 2018},
}

\bib{rosen2017hybrid}{inproceedings}{
      author={Rosen, Paul},
      author={Tu, Junyi},
      author={Piegl, L},
       title={A hybrid solution to calculating augmented join trees of 2d
  scalar fields in parallel},
        date={2017},
   booktitle={Cad conference and exhibition (accepted)},
}

\bib{shivashankar2012parallel}{article}{
      author={Shivashankar, Nithin},
      author={Senthilnathan, M},
      author={Natarajan, Vijay},
       title={Parallel computation of 2d morse-smale complexes},
        date={2012},
     journal={IEEE Transactions on Visualization and Computer Graphics},
      volume={18},
      number={10},
       pages={1757\ndash 1770},
}

\bib{singh2007topological}{inproceedings}{
      author={Singh, Gurjeet},
      author={M{\'e}moli, Facundo},
      author={Carlsson, Gunnar~E},
       title={Topological methods for the analysis of high dimensional data
  sets and 3d object recognition.},
        date={2007},
   booktitle={Spbg},
       pages={91\ndash 100},
}

\bib{snavsel2017geometrical}{article}{
      author={Sn{\'a}{\v{s}}el, V{\'a}clav},
      author={Nowakov{\'a}, Jana},
      author={Xhafa, Fatos},
      author={Barolli, Leonard},
       title={Geometrical and topological approaches to big data},
        date={2017},
     journal={Future Generation Computer Systems},
      volume={67},
       pages={286\ndash 296},
}

\bib{sumner2004deformation}{article}{
      author={Sumner, Robert~W},
      author={Popovi{\'c}, Jovan},
       title={Deformation transfer for triangle meshes},
        date={2004},
     journal={ACM Transactions on Graphics (TOG)},
      volume={23},
      number={3},
       pages={399\ndash 405},
}

\end{biblist}
\end{bibdiv}
